\documentclass[letterpaper, 10 pt, conference]{ieeeconf}  

\IEEEoverridecommandlockouts                              

\usepackage[dvipsnames]{xcolor}

\usepackage{cite}
\usepackage{comment}
\usepackage{tabularx}
\usepackage[latin1]{inputenc}
\usepackage{physics,amsmath}
\usepackage{amsfonts}
\usepackage{amssymb}
\usepackage{upgreek}
\usepackage{graphicx}
\usepackage{mathtools}
\usepackage{calc}
\usepackage{subcaption}
\usepackage{balance}
\usepackage{multirow}
\usepackage{hyperref}
\usepackage{booktabs}
\usepackage[linesnumbered, ruled,vlined]{algorithm2e}
\usepackage{tabularx}
\usepackage[ruled,vlined]{algorithm2e}
\usepackage{float}
\newtheorem{proposition}{Proposition}

\title{\LARGE \bf Automated Feature Selection for Inverse Reinforcement Learning }

\author{Daulet Baimukashev$^{1}$, Gokhan Alcan$^{2}$, Ville Kyrki$^{1}$
\thanks{This work was supported by the Academy of Finland under grant 347199.  (Corresponding author: Daulet Baimukashev)}%
\thanks{$^{1}$The authors are with the Intelligent Robotics Group, Department of Electrical Engineering and Automation (EEA), Aalto University, 02150, Espoo, Finland. (e-mail: {\tt\footnotesize daulet.baimukashev@aalto.fi, ville.kyrki@aalto.fi}) }
\thanks{$^{2}$The author is with the Faculty of Engineering and Natural Sciences, Automation Technology and Mechanical Engineering, Tampere University, Finland
(e-mail: {\tt\footnotesize gokhan.alcan@tuni.fi})}%
}

\begin{document}

\maketitle
\thispagestyle{empty}
\pagestyle{empty}

\begin{abstract}
Inverse reinforcement learning (IRL) is an imitation learning approach to learning reward functions from expert demonstrations. 
Its use avoids the difficult and tedious procedure of manual reward specification while retaining the generalization power of reinforcement learning.
In IRL, the reward is usually represented as a linear combination of features. 
In continuous state spaces, the state variables alone are not sufficiently rich to be used as features, but which features are good is not known in general.
To address this issue, we propose a method that employs polynomial basis functions to form a candidate set of features, which are shown to allow the matching of statistical moments of state distributions. 
Feature selection is then performed for the candidates by leveraging the correlation between trajectory probabilities and feature expectations. 
We demonstrate the approach's effectiveness by recovering reward functions that capture expert policies across non-linear control tasks of increasing complexity.
Code, data, and videos are
available at \url{https://sites.google.com/view/feature4irl}.
\end{abstract}


\IEEEpeerreviewmaketitle

\section{Introduction}

In the evolving landscape of robotics and machine learning, observational data has emerged as a cornerstone for learning and adapting intelligent behavior. It is often more straightforward for the experts to demonstrate a task than to explain it~\cite{ni2021f}. Imitation learning, the primary framework for learning from an expert, can be approached either by directly imitating the policy (behavioral cloning~\cite{torabi2018behavioral}) or by inferring the expert's intention or goal through learning the reward function (inverse reinforcement learning)~\cite{ng2000algorithms}). 

Behavioral cloning, however, suffers from a distributional shift between training and testing data, leading to erroneous actions in unseen states with errors that accumulate quadratically with the number of steps~\cite{xu2020error}. In contrast, inverse reinforcement learning (IRL) aims to recover a suitable reward function that explains expert behavior, subsequently deriving the optimal policy using gradient methods such as reinforcement learning~\cite{sutton2018reinforcement}. IRL offers the additional benefit of circumventing the challenge of manually designing appropriate reward functions for tasks~\cite{abbeel_apprenticeship_2004}, a process prone to incorrect assumptions about task and environment dynamics, potentially resulting in sub-optimal policies or task failure~\cite{DeepMind2020SpecificationGaming}.

\begin{figure}[t!]
  \centering
  \includegraphics[width=\linewidth]{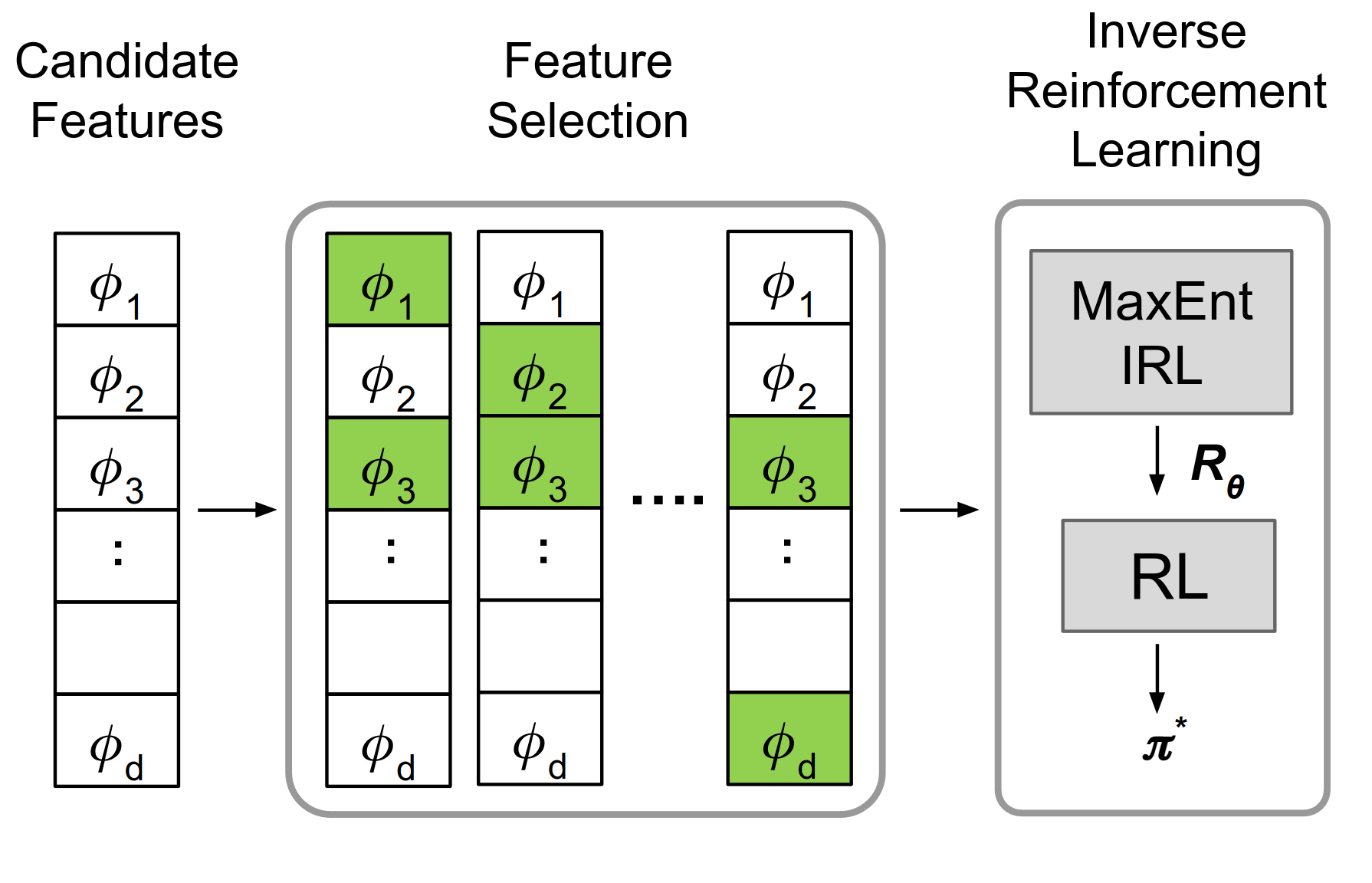}
  \caption{A central open challenge in inverse reinforcement learning is the choice of suitable features to represent the reward. We propose a method that constructs a candidate feature set and then selects a subset that best describes expected rewards.}
  \label{fig:framework}
\end{figure}

A systematic approach to formulating rewards in IRL for high-dimensional and/or continuous state space involves representing the reward function as a linear combination of relevant features~\cite{abbeel_apprenticeship_2004}. This formulation is flexible, as features can be nonlinear functions of states, but it necessitates identifying---usually manually---relevant state features to represent the reward function. Limited works in the literature explore methods for choosing suitable features. 

We propose to use polynomials as a candidate set for features. To limit the dimensionality, we furthermore propose an approach for selecting a subset of features that facilitates reward learning using regression methods. 

The primary contributions of the paper are:

\begin{enumerate}
    \item Showing that polynomial basis functions are effective as a candidate set of features, as they enable matching the statistical moments of the states between demonstrations and the retrieved policy,
    \item Developing an efficient feature selection mechanism that automatically selects the most relevant features through a correlation-based technique, favoring a smaller set of features to reduce reward complexity and mitigate the effects of noise and spurious correlations in IRL,
    \item Validating and evaluating the proposed feature generation and selection method by successfully retrieving the reward and corresponding expert policy for given expert demonstrations across multiple tasks of increasing complexity.
\end{enumerate}

\section{Related Works}

\textbf{Inverse reinforcement learning (IRL)} for inferring reward functions in continuous state spaces was pioneered by Ng et al.~\cite{ng2000algorithms} addressing the critical challenge of reward ambiguity where multiple reward functions can lead to the same optimal policy. This ambiguity was notably tackled by Ziebart et al.\cite{ziebart2008maximum} through a maximum entropy formulation that employs a probabilistic approach to reward determination, targeting a single policy that mirrors the training data distribution, albeit constrained to finite state spaces. The shift towards accommodating the continuous state spaces prevalent in robotics prompted approaches like state space discretization~\cite{boularias2011relative,qiao2011inverse,choi2011inverse}, path integral extensions~\cite{aghasadeghi_path_integral}, and more recent solutions leveraging orthonormal basis functions~\cite{dexter2021inverse}, Lyapunov theory~\cite{tesfazgi2021inverse}, and receding horizon strategies~\cite{xu2022receding} to tackle the complexity of continuous environments.

\textbf{Feature selection} for reward functions, traditionally a manual process reliant on domain expertise and characterized by trial-and-error, has posed significant challenges in terms of time and efficiency~\cite{phan2022driving,kuderer2015learning}. While Gaussian processes~\cite{levine2011nonlinear} and deep neural networks have been explored for feature learning, with adversarial methods~\cite{Finn2016AdversarialIRL, fu2017learning} emerging to mimic expert trajectories, such strategies face challenges related to data demands, training stability, and generalizability. The quest for compact, interpretable reward functions, less susceptible to adversarial issues, was partially addressed by Levine et al.~\cite{levine2010feature} through an iterative feature construction process. However, also this relied on a predefined set of basic functions or atomic features.

Our contribution distinctly advances the field by introducing an efficient algorithm that automates the selection of a concise set of features from a polynomial basis, addressing both the challenge of reward ambiguity and the inefficiencies in feature selection. This innovation not only streamlines the reward design process in IRL but also enhances the method's applicability and interpretability across various continuous state spaces, setting a new precedent for the development of intelligent systems in robotics.

\section{Background}
Reinforcement learning (RL) is an approach to solve Markov Decision Processes (MDP) defined as a tuple $\mathcal M = \langle \mathcal S, \mathcal A, \mathcal R, \mathcal T, \mathcal \gamma \rangle$, where $\mathcal S$ is the set of states, $\mathcal A$ the set of actions, $\mathcal R$ is the reward function, $\mathcal T$ represents transition dynamics, and $\mathcal \gamma$ is a discount factor. The solution represented by an optimal policy is defined as the one maximizing the expected discounted cumulative reward of the trajectory starting from any initial state $s$:
\begin{equation}\label{eq:pi_star}
\pi^*(s) = \underset{a}{\arg\max} \ \mathbb{E} \left[ \sum_{t=0}^{\infty} \gamma^t R_{t+1} \bigg| S_0 = s, A_0 = a \right]
\end{equation}
In the RL setting, the transition dynamics are unknown, and the policy is optimized through exploration (trial and error), observing states, actions, and rewards.

In contrast, we consider \emph{inverse reinforcement learning} (IRL) where the objective is---in the same setting---to infer a reward function that is compatible with given expert state-action trajectories $\tau=(s_0, a_0, \ldots, s_N, a_N)$ forming a dataset $D$ = $\{\tau_1, \tau_2, \ldots, \tau_k \}$. IRL is a fundamentally ill-defined (under-constrained) problem in that infinitely many reward functions are compatible with any dataset.

To constrain the problem, representing the reward as a linear combination of features $\phi(s)=(\phi_0(s),\ldots,\phi_d(s))$ has been proposed~\cite{ng2000algorithms} so that
\begin{equation}
\label{r_s}
R(\vb{s}) = \vb{\theta}^T \phi(\vb{s})
\end{equation}
where $\vb{\theta}$ represents the weight of each feature. Due to the linearity, the cumulative reward of a trajectory can then be defined as 
\begin{equation}\label{eq:r_tau}
    R(\tau) = \vb{\theta}^T \phi(\tau),
\end{equation} where $\phi(\tau) = \sum_{\vb{s_i} \in \tau}^{}  \phi(\vb{s_i})$. With this formulation, the optimal policy must match the feature expectations with training data~\cite{abbeel_apprenticeship_2004}, that is,
\begin{equation}\label{eq:exp_data}
    E_{D}[\phi(s)]=E_{\pi} [\phi(s)].
\end{equation}

Now, the focus of our work is on choosing which set of features is used given a training dataset. 

\section{Method}

\begin{algorithm}[t] 
\caption{Inverse Reinforcement Learning with Feature Selection}\label{alg:cap}

\SetAlgoLined
\SetKwInOut{Input}{input}\SetKwInOut{Output}{output}

\Input{$\mathcal M \backslash r$, expert data $D$, simulator $E$, epoch $M$, learning rate $\alpha$, number of trajs $N$}
\Output{Reward $r$, policy $\pi$}

$\mu_e \leftarrow  \text{compute\_features}(D)$

Generate candidate feature set $\Phi$ \tcp{\small{(Sec. \ref{sec:cand_feats})}} 
Fit kernel density function to $D$

Generate labels  $H$=$\{\log P(\tau_i) \mid i = 0, 1, 2, \ldots, N\}$

$\phi(\cdot) \leftarrow \text{rank\_features}(\Phi, H)$ \tcp{\small{(Sec. \ref{sec:feat_select})}}

Initialize $\theta$ $\sim Unif[-1, 1]$


\For{$i=0$ \KwTo $M$}{


    $r \leftarrow \text{construct\_reward}(\theta, \phi)$ 
    
    Configure simulator $E$ with new $r$

    $\pi \leftarrow \text{learn\_policy}(E, r)$ \tcp{\small{(Sec. \ref{policy_extract})}}
    
    $G \leftarrow \text{collect\_rollouts}(\pi, E, N)$
    
    $\mu_a \leftarrow  \text{compute\_features}(G)$

    $\nabla L(\theta) = \mu_e - \mu_a$ \tcp*{\small{loss gradient}}

    $\theta \leftarrow \theta - \alpha \nabla L(\theta)$ \tcp*{\small{update $\theta$}}
    

}

\end{algorithm}

In this section, we present the proposed method for inverse reinforcement learning, focusing on the choice of features. 
The entire approach is described in Algorithm 1. 
We begin by proposing to use polynomials as candidate features (lines 1-2, Sec. IV-A). Next, we present the feature selection method (line 3, Sec IV-B), which selects subset of relevant features.
We then present how the feature weights can be optimized through maximum entropy IRL (Sec. IV-C) using regular RL as a component (Sec. IV-D).

\subsection{Polynomial features }\label{sec:cand_feats}

We propose to use quadratic polynomials of the state as candidate features. Thus,
\begin{equation}\label{eq:phi_s}
    \phi(\mathbf{s})=
        \begin{pmatrix}\mathbf{s}\\
        vec(\mathbf{s} \mathbf{s}^T)
    \end{pmatrix}
\end{equation}
where $vec(\cdot)$ denotes the vectorization of a matrix. To argue for the choice, we next show that this corresponds to matching the statistical moments of the state up to second order (mean, covariance) between demonstrations and retrieved policy.

\begin{proposition}
Matching the expectations of features consisting of second-order polynomials leads to matching the mean and variance of the distributions. 
\end{proposition}

\begin{proof}

Substituting $\phi(s)$ from \eqref{eq:phi_s} to \eqref{eq:exp_data} gives
\begin{equation}\label{eq:exp_phi}
    E_{D} \begin{bmatrix}\mathbf{s}\\
                vec(\mathbf{s} \mathbf{s}^T)
            \end{bmatrix}=E_{\pi}\begin{bmatrix}\mathbf{s}\\
                                        vec(\mathbf{s} \mathbf{s}^T)
                                    \end{bmatrix}.
\end{equation}

Taking the expectations componentwise this reduces to 
\begin{subequations}
\label{eq:exp_phi2}
\begin{align}
     E_{D} [\mathbf{s}]&=E_{\pi} [\mathbf{s}] \label{eq:means}\\
     E_{D} [vec(\mathbf{s}^T \mathbf{s})] &=
        E_{\pi} [vec(\mathbf{s} \mathbf{s}^T)].\label{eq:vars}
\end{align}
\end{subequations}
The means of state distributions being equal under dataset and policy follows now trivially from \eqref{eq:means}.
Now, note that the covariance matrix of the state is defined as 
\begin{equation}
cov(\mathbf{s},\mathbf{s})=E[\mathbf{s}\mathbf{s}^T]-E[\mathbf{s}]E[\mathbf{s}]^T.    
\end{equation}
Now, the first term of the right-hand side is equal under dataset and policy due to \eqref{eq:vars} and the second term due to \eqref{eq:means}. Then, the covariance matrices of state distributions are equal under dataset and policy.  

Thus, using polynomials up to second order as features matches the mean and covariance of state distributions for the training dataset and recovered policy. 
\end{proof}

Matching mean and covariance can also be interpreted as matching the Gaussian approximations of state distributions. This motivates the choice, because Gaussians are maximum entropy distributions making the least assumptions about the underlying distribution under known mean and covariance. The dimensionality of the feature set remains also relatively small, $dim(\Phi)=d+d(d+1)/2$ where $d=dim(\mathbf{s})$. However, for a high-dimensional state space, using this many features can still be problematic.

\subsection{Feature selection}\label{sec:feat_select}

A smaller set of features is preferred as it reduces the reward complexity and helps to avoid the effect of noise and pseudo-correlation. Now, we present an efficient method to select only relevant ones.

As shown in~\cite{ziebart2008maximum}, the probability of the trajectory in the expert dataset is proportional to the exponential of the reward of that trajectory, that is,
\begin{equation}
\label{eq_ptau}
    P(\tau_i) | \theta) = \frac{e^{\theta^T  \phi(\tau_i)}}{Z(\theta)}
\end{equation}
where partition function $Z(\theta)$ normalizes the reward function over all possible trajectories. 

Taking logarithm of both sides of  \eqref{eq_ptau}, we get
\begin{equation}
\label{log_tau}
    \log P(\tau_i) | \theta) = \theta^T  \phi(\tau_i) - \log Z(\theta),
\end{equation}
and noting that the partition term $Z$ is constant, it can be omitted: 
\begin{equation}
\label{log_tau_prop}
    \log P(\tau_i) | \theta) \propto \theta^T  \phi(\tau_i).
\end{equation}
Thus, the log probability of trajectories is proportional to a linear combination of features. 

Now, we propose to estimate the left-hand side of \eqref{log_tau_prop}, the probability of trajectories, from the empirical training data, for each trajectory, as follows. To avoid making assumptions on the unknown transition dynamics of the MDP, the probability of the trajectory is expanded into a product of individual state probabilities
\begin{equation}
    P(\tau) = P(s_1) P(s_2) \ldots P(s_n).
\end{equation}

To estimate the probability of individual states $P(s_t)$, we perform kernel density estimation using the entire training dataset. 
The kernel density estimate for state $s$ is defined as
\begin{equation}
\hat{P}(\mathbf{s}) = \frac{1}{|D|} \sum_{\mathbf{t}\in D} K(\mathbf{s} - \mathbf{t}).
\end{equation}
We use the Gaussian kernel function defined as 
\begin{equation}
    K(\mathbf{s}) = \frac{1}{(2\pi)^{d/2} |\Sigma|^{1/2}} \exp\left(-\frac{1}{2} \mathbf{s}^T \Sigma^{-1} \mathbf{s}\right)
\end{equation}
where $\Sigma$ is the covariance parameter related to the kernel width.

After computing the probabilities of trajectories, i.e. left-hand side of the equation \ref{log_tau_prop}, for each trajectory in the training dataset, we construct a set containing labels $H$=$\{\log P(\tau_i) \mid i = 0, 1, 2, \ldots, N\}$. Next, we aim to find which of the candidate features have higher predictive powers for trajectory probability. One way to perform feature selection is to utilize a univariate feature selection method based on statistical tests. In our case, to rank features by their importance we use F-statistics between trajectory features and their probabilities from $H$. Then, features with higher F-statistics are selected for feature extractor $\phi(\cdot)$.

The time complexity of the feature selection is $\mathcal{O}(N)$ w.r.t. feature size. This algorithm is simple yet efficient, as we do not consider many combinations of features as in the case of manual feature exploration. The next step to fully recover the reward function is to find the weights of each of the features which is covered in the next part.

\subsection{Reward retrieval}

We apply maximum entropy IRL~\cite{ziebart2008maximum} to learn the weights of selected features. To maximize the log probability of the observed data, we write the equation as follows: 

\begin{equation}
\theta^* = \underset{\theta}{\arg\max} \sum_{\tau \in D }^{} \log P(\tau | \theta)
\end{equation}

If we take the derivative of the log-likelihood, the loss function becomes

\begin{equation}\label{grad_cost}
    \nabla L(\theta) = \mu_e - \sum_{i=1}^{m} \phi(\tau'_{i}) = \mu_e - \sum_{i=1}^{m} \sum_{\vb{s_i} \in \tau}^{}  \phi(\vb{s_i}) 
\end{equation}
where $\mu_e$ is the feature expectation of the training data. The weights of the reward can be updated by using the gradient descent algorithm as follows:

\begin{equation}\label{theta_update}
  \theta  \leftarrow \theta - \alpha  \nabla L(\theta)
\end{equation}
where $\alpha$ is the learning rate.

\subsection{Policy extraction}\label{policy_extract}
We use a reinforcement algorithm to extract the expert policy that maps states to actions by maximizing the given reward function. The algorithm finds an optimal policy $\pi^*$ using the formulation from Eq. \ref{eq:pi_star}. Particularly, we use the proximal policy optimization (PPO) method from~\cite{schulman2017proximal} and soft-actor-critic (SAC)~\cite{haarnoja2018soft} algorithm. PPO is a prevalent method in many cases due to lower sensitivity to hyperparameters, while SAC is sample-efficient and works with continuous action spaces.

\section{Experiments}

In the experiments, we address the following research questions:
\begin{itemize}
  \item Is the candidate set of basis functions using polynomials sufficient to solve nonlinear control tasks?
  \item  What is the performance of inverse reinforcement learning compared to reinforcement learning with known reward?
  \item What is the performance of the proposed method compared to baselines?

\end{itemize}

\subsection{Setup}
We conduct experiments using the following three Gymnasium~\cite{towers_gymnasium_2023} environments: Pendulum-v1, CartPole-v1, and Acrobot-v1 with an increasing number of states space. The Fig.\ref{fig:agents} shows the snapshots of the environments and size of the candidate feature set. Below are short descriptions of the task and the \textit{true} reward functions:
\begin{enumerate}
    \item \textbf{Pendulum} \\ 
        The task is to swing upwards and stabilize the pendulum at vertical position. The \textit{true} reward is the negative sum of the square of the angle, angular velocity, and torque applied.
    \item \textbf{CartPole} \\ 
        The task is to stabilize the pole mounted on a cart and keep the positon of the cart closer to the center of the scene. The \textit{true} reward is a scalar value {+1} for each step that the pole is upright. 
    \item \textbf{Acrobot} \\ 
        The task is to reach certain height threshold with two-link system. 
        The \textit{true} reward is a scalar value {-1} for each step for not reaching the target. \\
\end{enumerate}

\begin{figure}[!t]
\centering
\captionsetup{justification=centering}

\begin{subfigure}{.32\columnwidth}
  \centering
  \includegraphics[height=0.1\textheight]{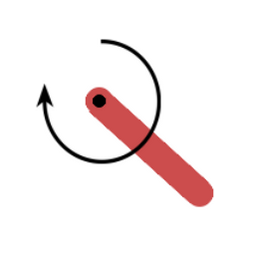}
  \caption{Pendulum-v1.\\ \textit{dim($\Phi$)} = 9}
\end{subfigure}
\begin{subfigure}{.32\columnwidth}
  \centering
  \includegraphics[height=0.1\textheight]{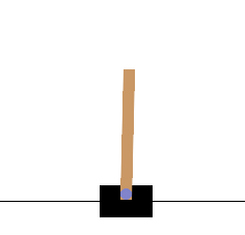}
  \caption{CartPole-v1.\\ \textit{dim($\Phi$)}  = 14}
\end{subfigure}
\begin{subfigure}{.32\columnwidth}
  \centering
  \includegraphics[height=0.1\textheight]{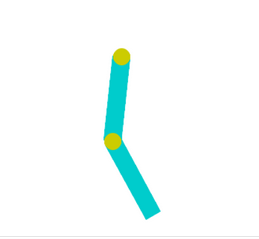}
  \caption{Acrobot-v1. \\ \textit{dim($\Phi$)}  = 27}
\end{subfigure}

\caption{Benchmark tasks used in this paper.}
\label{fig:agents}
\vspace*{-15pt}
\end{figure}

\subsection{Data collection}
To collect the expert or training data, we train RL algorithms for these three environments. Particularly, we used SAC for the Pendulum environment with a continuous action space, and PPO was chosen for CartPole and Acrobot with discrete action space. We train RL algorithms to maximize the cumulative \textit{true} reward function of the environment. After reaching established benchmark results as presented in Stable Baselines3~\cite{stable-baselines3} for these tasks, we refer to the RL policy as an expert policy and deploy it for data collection. For each of the environments, we collected $N$ number of trajectories, and the starting states of the trajectories were randomly sampled, and simulated trajectories were saved in dataset~$D$. Thus, in our IRL method as shown in Algorithm~\ref{alg:cap}, only dataset $D$ is used, and we assume that expert policy or true reward function is unknown.

\begin{table}[b]
\centering
\begin{tabular}{|l|l|l|l|}
\hline
 & Pendulum & Cartpole & Acrobot \\ \hline
Number of expert trajs & 200& 200& 100\\ \hline
Epochs & 100& 100& 100\\ \hline
Learning rate & 0.2& 0.2& 0.2\\ \hline
Learning decay & 0.97& 0.97&0.97\\ \hline
Total timestep & 8e4 & 1e5 & 0.5e6 \\ \hline
Gamma & 0.9 & 0.97 & 0.99 \\ \hline
Number of envs & 4 & 8 & 4 \\ \hline
Batch size & 64 & 128 & 128 \\ \hline
Number of nodes (MLP) & 64x2 & 96x2 & 400x2 \\ \hline

\end{tabular}
\caption{Training hyperparameters.}
\label{tab:hyperps}
\end{table}

\begin{figure*}
    \centering
    \includegraphics[width=0.95\linewidth]{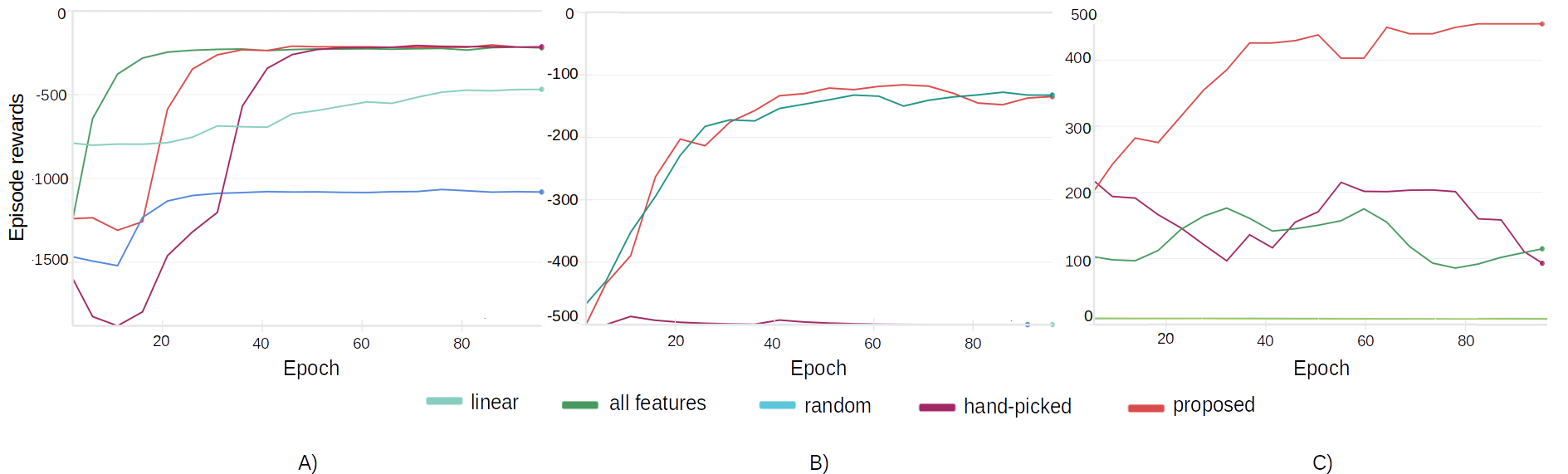}
    \caption{Mean cumulative rewards for policies trained using various feature sets, calculated across 10 different initial conditions. A) Pendulum, B) Acrobot, C) CartPole.}
    \label{fig:episode_rewards}
\end{figure*}

\begin{table*}[!ht]
\centering
\begin{tabularx}{0.9\linewidth}{|l|X|X|X|X|X|X|}
\hline
 & linear & all features & random & hand-picked & \textbf{proposed} & SB3 benchmark \\ \hline
Pendulum & -467.78 & \textbf{-220.81} & -846.87 & \textbf{-214.13} & \textbf{-216.84} & -156.99 $\pm$ 88.7\\ \hline 
Cartpole & 8.4 & 114.71 & 8.41 & 92.86 & \textbf{454.9} & 500.00 $\pm$ 0.0\\ \hline
Acrobot & -500 & \textbf{-132.32} & -499.83 & -500 & \textbf{-134.64} & -73.50 $\pm$ 18.20\\ \hline
\end{tabularx}
\caption{Mean cumulative rewards for policies trained using various feature sets, calculated across 10 test simulations under varying initial conditions. The last column shows the Stable-Baseline3 (SB3) benchmark for the RL policy with \textit{true} reward of the environment. The cells with bold values indicate satisfactory performance.}
\label{res}
\end{table*}

\subsection{Baselines}

In this work, we propose a feature selection mechanism aimed at recovering the reward function that corresponds to the expert's behavior, given the dataset. Therefore, we compare our method against various baseline feature selection strategies: \textit{hand-picked} features using domain expertise and refined through trial-and-error, \textit{randomly} selected features, direct use of states as features (referred to as \textit{linear} features), and inclusion of \textit{all} features in the candidate set. Given that comparing reward functions directly across baselines does not yield a meaningful performance metric in terms of completing a task, we focus on comparing the \textit{policies derived} from these reward functions. To this end, we conduct two comparative analyses. First, we execute the derived policies in multiple testing environments configured by \textit{true} reward functions and observe cumulative rewards. This evaluation is justified because our training data comes from an expert trained with \textit{true} reward function as well. Second, we compare the state distributions of the training data from expert and testing data from the extracted policies. For the divergence measure of multivariate distributions, we exploited the 2D Wasserstein distance metric~\cite{Villani2008}. 

\subsection{Implementation details}

The source code was implemented in Python 3.8. and we have made the code publicly available\footnote{\label{foot:git}https://sites.google.com/view/feature4irl}. We utilized the Stable-baselines3 library~\cite{stable-baselines3} for the training of RL algorithms. As a policy network, we use a multi-layer perceptron (MLP) with two hidden layers. Furthermore, we conducted a thorough hyperparameter optimization for both the RL and inverse reinforcement learning (IRL) parameters. Table\ref{tab:hyperps} summarizes the hyper-parameters used during the training. The total number of training epochs for IRL is 100, and the training dataset consists of 200 trajectories. To facilitate the training, we parallelized data collection and environment simulation using multi-processing techniques. The training was performed using an Aalto high-performance computing cluster.

\subsection{Results}

Figure~\ref{fig:episode_rewards} illustrates the mean cumulative rewards of the 10 environment simulations for each of the tasks. Each line represents the performance of the retrieved policy using the corresponding set of features for the reward functions. We observe that our proposed method (indicated by the red line) achieves benchmark results across all tasks. It is also evident that manually selected features perform sub-optimally in two tasks, namely CartPole and Acrobot. Despite multiple iterations and attempts to manually identify the best features representing the expert reward, this task proved to be non-trivial. For Pendulum and Acrobot, Inverse RL successfully identifies suitable reward functions using all features, although employing all features does not always ensure success, as demonstrated by the CartPole task. This discrepancy may be attributed to noise or spurious correlations between features and rewards, complicating the learning process. In all three environments, neither random nor first-order state features achieved a satisfactory performance level.

Nonetheless, our method attains comparable performance levels using significantly fewer features. Notably, for the Pendulum task, our method reaches benchmark results more swiftly than methods based on hand-picked features.

Subsequently, Table \ref{res} provides a quantitative comparison of our method and other baseline results against benchmark scores for the tasks in question by Stable-Baselines3 (an RL policy employing the \textit{true} reward function). We see that the proposed method achieves sufficient benchmark results in all tasks. The selection of all features succeeded in only two cases, while hand-picked features performed well in just one case. Additionally, the performances of the successful policies have been visually confirmed and the corresponding video is included in the supplementary materials. 

Figure \ref{fig:wasser} displays the 2D Wasserstein distance between expert data and testing data from the retrieved policy for the Pendulum and Acrobot environments. Although the cumulative rewards for reward functions utilizing all features and our proposed features are similar (Table~\ref{res}), the Wasserstein distance for the proposed method is considerably lower than that for employing all features in both of the tasks. This suggests a closer alignment in state distribution with the expert data. As previously mentioned, the observed discrepancy could stem from noise or spurious correlations between certain features and the rewards. This finding highlights the benefit of employing a smaller, more concise set of features for improved performance and alignment with expert behavior.

\section{Discussion}\label{sec:concl}

In this study, we introduced algorithms for generating candidate features and selecting them automatically. Our findings reveal that second-order polynomial basis functions perform effectively as feature extractors for tasks involving continuous state control. We posit that employing higher-order polynomials could align not only the mean and variance but also the higher-order statistical moments between training and testing distributions, thereby enhancing the match between these distributions.

Basis functions offer an effective and versatile method for creating complex features. Within the scope of this paper, we considered only polynomials but other basis functions could be easily integrated into our method. Indeed, we believe that our approach is compatible with different basis functions, including but not limited to radial-basis functions (RBFs)\cite{sutton2018reinforcement} and Fourier series~\cite{dexter2021inverse,konidaris2011value}. Our approach capitalizes on the correlation between the trajectory's probability and its feature expectation, allowing for the inclusion of a broad spectrum of feature functions. We anticipate that our algorithm can be effectively expanded to integrate a mixture of different basis functions. Notably, even with the addition of diverse basis functions, the complexity of our algorithm would increase linearly with the number of features.
    
\begin{figure}[t]
    \centering
        \centering
        \includegraphics[width=\linewidth]{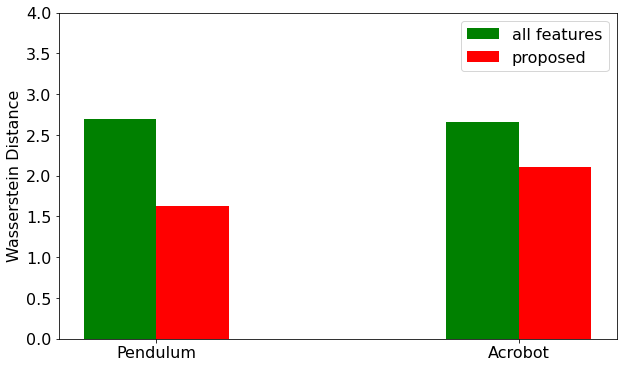}
        \label{fig:sub1}
    \hfill
    \caption{2D Wasserstein distance between training and testing data for the Pendulum and Acrobot environments.}
    \label{fig:wasser}
\end{figure}

Employing a more concise set of features enhances the generalization and robustness of inferred rewards by minimizing the impact of noise and spurious correlations. We believe that another benefit of using a compact set of features is a reduction in the complexity of interpreting rewards and getting insights into how the policy is trying to solve the task. 

We anticipate that the method introduced here could also be beneficial for tasks such as preference learning, especially in scenarios where the task setup remains the same but involves different experts. In such instances, we believe that the behavior of various experts might be described by a common set of features. However, the distinction in their modes of behavior would be reflected through the differing weights assigned to these features.

\section{Conclusion}\label{sec:concl2}
In this study, we introduce an efficient feature selection algorithm that utilizes polynomial basis functions for constructing reward functions in inverse reinforcement learning, demonstrating its effectiveness across three distinct environments. By automatically selecting the most relevant features from a candidate set, our method simplifies the reward learning process, enhancing model interpretability and the fidelity of expert behavior replication. Highlighting the potential for further refinement, we plan to explore alternative basis functions, including Fourier series, radial basis functions, and higher-order polynomials, to expand our approach's applicability and precision. This exploration is anticipated to enhance the adaptability and accuracy of our feature selection algorithm, offering new avenues for refinement in the field of imitation learning.

\section*{ACKNOWLEDGMENT}
The authors would like to acknowledge the computational resources provided by the Aalto Science-IT project and Aalto Research Software Engineering service. 


\bibliographystyle{IEEEtran}
\bibliography{main}

\end{document}